\documentclass[letterpaper, 10 pt, conference]{ieeeconf}  

\IEEEoverridecommandlockouts                              

\title{\LARGE \bf
Robust Peak-cost Constrained Reinforcement Learning
}

\author{Shilpa Mukhopadhyay$^{1}$, Sourav Ganguly$^{1}$, Santosh Mohan Rajkumar$^{3}$,\\ Honghao Wei$^{2}$, Debdipta Goswami$^{3}$ and Arnob Ghosh$^{1}$ 
\thanks{$^{1}$Dept. of Electrical and Computer Engineering, NJIT, Newark, NJ, USA
        {\tt\small (E-mails: \{sm3934,sg2786,arnob.ghosh\}@njit.edu)}}%
\thanks{$^{2}$School of Electrical Engineering and Computer Science
Washington State University, Pullman, WA, USA
        {\tt\small (E-mail: honghao.wei@wsu.edu)}}%
\thanks{$^{3}$ Department of
Mechanical and Aerospace Engineering, The Ohio State University, Columbus,
OH, USA
        {\tt\small (E-mail: \{rajkumar.36,goswami.78\}@osu.edu)}}%
}

\usepackage{amsmath}
\usepackage{amssymb}
\usepackage{cite}
\usepackage{hyperref}
\usepackage[ruled,vlined]{algorithm2e}
\usepackage{xcolor}
\usepackage{graphicx}
\usepackage{titlesec}
\usepackage{subcaption}
\let\labelindent\relax
\usepackage{enumitem}

\newcommand{\parab}[1]{\noindent\textbf{#1}\ }
\SetKw{KwParam}{Parameters:}

\newtheorem{proposition}{Proposition}
\newtheorem{theorem}{Theorem}
\newtheorem{remark}{Remark}

\newcommand{\ours}{RP-CRL}

\titlespacing*{\subsection}
  {0pt}{0.4\baselineskip}{0.2\baselineskip}

 \titlespacing*{\section}
   {0pt}{0.7\baselineskip}{0.3\baselineskip}

\renewcommand{\baselinestretch}{0.97}
\allowdisplaybreaks

\usepackage{xcolor}

\begin{document}
\bstctlcite{IEEEexample:BSTcontrol}

\maketitle
\thispagestyle{empty}
\pagestyle{empty}

\begin{abstract}
We study robust peak-cost constrained reinforcement learning  (\ours), where the objective is to maximize expected reward while controlling the maximum cost encountered along a trajectory. This setting is motivated by safety-critical applications in which a single large violation can be catastrophic and therefore cannot be adequately captured by the standard CMDP framework based on expected cumulative cost. Existing reachability-constrained RL methods adopt Lagrangian-based approaches, yet the underlying duality properties of peak-cost constrained MDPs remain unclear. We show that, unlike standard CMDPs, peak-cost constrained MDPs may not admit zero duality gap. We further consider a robust formulation to address simulator-to-real-world mismatch in the transition dynamics. To solve this problem, we develop a surrogate optimization framework and a robust value estimation method based on integral probability metrics. We prove that, with appropriate hyperparameter choices, the surrogate solution attains the same robust reward value as the original problem while violating the constraint by at most \(\epsilon\). Experiments show that the proposed method effectively enforces safety under dynamics perturbations while retaining strong reward performance across diverse environments.
\end{abstract}



\vspace{-0.05in}
\section{Introduction}
\vspace{-0.05in}
Safety is a fundamental requirement in intelligent decision-making problems. Recent reinforcement learning (RL) research has extensively studied constrained Markov decision processes (CMDPs) \cite{altman2021constrained,ghosh2024towards,ghosh2023achieving,tesslerreward}. In the standard CMDP framework, safety is enforced by requiring the \emph{expected cumulative cost} to remain below a prescribed threshold. However, this formulation is insufficient for many practical applications, where the primary concern is not the total accumulated cost, but rather the \emph{peak cost incurred along a trajectory}. In such settings, even a single large violation can be catastrophic, regardless of how small the cumulative cost is overall. For example, a trajectory may incur small costs at most time steps but experience one severe safety violation, making it significantly more dangerous than another trajectory whose maximum cost remains uniformly bounded.

To address this limitation, recent works \cite{yu2022reachability,ganai2023iterative} have considered reachability-constrained reinforcement learning (RCRL), where the goal is to maximize expected reward subject to a peak-cost constraint, namely that the maximum cost encountered along a trajectory remains below a specified threshold (often \(0\)). These works develop Lagrangian-based methods for solving the resulting optimization problem. However, unlike the standard CMDP setting, where zero duality gap is well established, such a guarantee has not been shown for peak-cost constrained MDPs. As a result, it remains unclear whether these Lagrangian-based approaches can recover the true optimal solution.

An additional limitation of existing work is that it does not account for uncertainty in the transition dynamics \cite{yu2022reachability,ganai2023iterative}. In practice, policies are often trained in simulation and then deployed in the real world, where the true dynamics may differ from those of the simulator. Consequently, a policy that satisfies the peak-cost constraint in simulation may still violate the same constraint after deployment due to model mismatch. This motivates the following central questions: 
\vspace{-0.04in}
\begin{center}
{\em Does a peak-cost constrained MDP admit zero duality gap? How can we develop an algorithm for the robust peak-cost constrained problem that bridges the sim-to-real gap?}
\end{center}
\vspace{-0.04in}

\textbf{Our Contributions:}
\begin{itemize}[leftmargin=*]
    \item We show that, unlike the standard CMDP, a peak-cost constrained MDP may fail to admit zero duality gap, even in a two-state MDP. The key reason is that convexity with respect to the state-action occupancy measure breaks down unlike the CMDP.

    \item We formulate a robust peak-cost constrained problem and propose a surrogate optimization framework. We show that, with appropriate hyperparameter choices, the optimal solution of the surrogate problem achieves the same robust reward value as the original problem while violating the constraint by at most an \(\epsilon\)-amount.

    \item We develop an effective approach for estimating the robust value function based on an integral probability metric. Empirically, we show that the proposed method can effectively satisfy the constraint even when the environment is perturbed.
\end{itemize}


\subsection{Related Works}
\textit{\parab{CMDP}}:
Constrained Markov Decision Processes (CMDPs) are a foundational framework for solving constrained reinforcement learning (CRL) problems~\cite{altman2021constrained}. CMDPs leverage the convexity of the state-action occupancy measure, enabling the use of primal-dual methods to solve the optimization problem~\cite{paternain2022safe, stooke2020responsive, tesslerreward, zheng2020constrained}. These methods have well-established convergence guarantees and have been extensively studied in the literature~\cite{ghosh2023achieving, ding2020natural, ghosh2024towards}. Beyond primal-dual methods, linear programming (LP)-based and model-based approaches have also been explored, directly addressing the primal problem~\cite{xu2021crpo, achiam2017constrained, chow2018lyapunov}. However, peak-constrained objective is different compared to the CMDP.

\textbf{RCMDP}:
 Unlike standard CMDPs, the optimization problem for Robust Constrained MDPs (RCMDPs) is non-convex in the state-action occupancy measure, which complicates the use of traditional Lagrangian or primal-dual methods~\cite{wang2022robust}. Some studies have attempted to address this challenge using primal-dual methods under strong duality assumptions, but these approaches often lack iteration complexity guarantees~\cite{zhang2024distributionally}. Recent work has introduced epigraph reformulations to address the limitations of primal-dual methods, but these approaches require computationally expensive binary search procedures and are sensitive to noise in policy value estimation~\cite{kitamuranear}. \cite{gangulyefficient} avoids these limitations by jointly optimizing for robustness and constraint satisfaction, eliminating the need for binary search with iteration guarantees. However, these methods do not consider peak cost constraint which we consider.

\textbf{Feasible Sets and Reachability Analysis for Safety and Optimality in CMDPs}:
Characterizing feasible sets is a critical and ongoing challenge in safe control and reinforcement learning~\cite{brunke2022safe}. Feasible sets, often represented using safety certificates like control barrier functions (CBFs)~\cite{choi2021robust, choi2021robust, luo2021learning}, define the set of states from which the system can recover to safety. However, such safety certificates often result in overly conservative or required to know the dynamics, limiting their applicability~\cite{ma2022joint}. Hamilton Jacobian (HJ) reachability analysis provides a more rigorous approach to deriving feasible sets but is computationally expensive due to the non-trivial partial differential equations involved~\cite{bansal2017hamilton}. Recently, RL-based approaches are considered to find policy that will avoid unsafe sets \cite{fisac2018general,fisac2019bridging,so2024solving}, however, they did not consider on maximizing reward as an objective. \cite{yu2022reachability,ganai2023iterative} learns a single policy that balances safety and optimality,considering peak cost constraint optimization. They used Lagrangian framework to solve the problem, however, they did not provide any strong duality guarantee. Further, none of the above work considered robust optimization framework bridging the sim-to-real gap. 
\vspace{-0.04in}
\section{Preliminaries and Problem Formulation}
\vspace{-0.03in}
\subsection{Constrained Markov Decision Processes}
A Constrained Markov Decision Process (CMDP) is defined by
\[
M=\langle \mathcal{S},\mathcal{A},P,r,c,\gamma,d_0\rangle,
\]
where \(\mathcal{S}\) and \(\mathcal{A}\) denote the state and action spaces, \(P(\cdot\mid s,a)\) is the (possibly unknown) transition kernel, \(r:\mathcal{S}\times\mathcal{A}\to\mathbb{R}\) is the reward function, \(c:\mathcal{S}\times\mathcal{A}\to\mathbb{R}\) is the cost function, \(\gamma\in(0,1)\) is the discount factor, and \(d_0\) is the initial state distribution. A policy is denoted by \(\pi(\cdot\mid s):\mathcal{S}\rightarrow \Delta(\mathcal{A}) \).

For a given transition model \(P\) and policy \(\pi\), let \(V_r^{P,\pi}(s)\) and \(V_c^{P,\pi}(s)\) denote the discounted negative reward and cost value functions, respectively:
\[
V_r^{P,\pi}(s)
=
-\mathbb{E}\!\left[ \sum_{t=0}^{\infty}\gamma^t r(s_t,a_t)\,\middle|\, s_0=s,\pi,P\right],
\]
\[
V_c^{P,\pi}(s)
=
\mathbb{E}\!\left[\sum_{t=0}^{\infty}\gamma^t c(s_t,a_t)\,\middle|\, s_0=s,\pi,P\right].
\]
The corresponding objective values from the initial distribution \(d_0\) are
\[
J^P_r(\pi)=\mathbb{E}_{s\sim d_0}\!\left[V_r^{P,\pi}(s)\right],
\quad
J^P_c(\pi)=\mathbb{E}_{s\sim d_0}\!\left[V_c^{P,\pi}(s)\right].
\]

A standard CMDP seeks to maximize reward while keeping the expected discounted cost below a budget \(b\). The CMDP objective becomes:
\[
\min_{\pi}\; J^P_r(\pi)
\qquad
\text{s.t.}\qquad
J^P_c(\pi)\le b.
\]
When Slater's condition  holds, the standard CMDP admits zero duality gap \cite{paternain2019constrained}, and therefore Lagrangian-based methods provide an effective solution approach. {\em Without loss of generality, we assume that $-r\in [0,1]$.}
\vspace{-0.05in}
\subsection{Peak-Cost Constraints}
Although CMDPs are widely used, a major limitation is that the constraint is expressed only through the \emph{expected cumulative discounted cost}. In many applications, however, controlling only the cumulative cost is insufficient; instead, one must ensure that the \emph{maximum} cost encountered along a trajectory remains below a desired threshold.

\textbf{Concrete example (reachability and persistent safety)}
In many safety-critical settings, we require that the system never enter an unsafe region. Let \(c(s)\) be a state-based safety function such that \(c(s)\le b\) for safe states and \(c(s)>b\) for unsafe states. Then, for a policy \(\pi\), safety along a trajectory \(\tau=(s_0,s_1,\dots)\sim\pi\) requires
$\max_{t\in\mathbb{N}} c^{\pi}(s_t)\le 0$ with probability $1$ where $b$ is some safety threshold.
That is, the maximum safety violation encountered along the trajectory must remain non-positive.

Motivated by this, we consider the following peak-constrained problem:
\[
\min_{\pi}\; J_r^P(\pi)
\qquad
\text{s.t.}\qquad
\max_{s\in \mathcal{S}^{\pi,P}} c(s)\le b,
\]
where $\mathcal{S}^{\pi,P}=\{s\in \mathcal{S}:\Pr(\exists t\geq 0, s_t=s|s_0\sim d_0,\pi,P)>0\}$. We assume that the initial states always give you at least one policy which is feasible. Often times, we consider an episodic setup with length $T$ then $t$ will be upper bounded by $T$.

Note that the peak constraint is difficult to evaluate as one needs to predict the future states at a given time $t$, and can be pessimistic. Inspired from \cite{fisac2018general} for a fixed policy \(\pi\), we consider an expected peak-cost constraint function using a discounted recursion. In particular, define
\begin{align}
Q_{c,\mathrm{peak}}^{\pi,P}(s,a) &= (1-\gamma)c(s) \notag \\
& + \gamma \max\left\{c(s),\,\mathbb{E}_{P(\cdot\mid s,a)}\left[V_{c,\mathrm{peak}}^{\pi,P}(\cdot)\right]\right\} \label{eq:max_value}
\end{align}
and
\[
V_{c,\mathrm{peak}}^{\pi,P}(s)=\langle \pi(\cdot\mid s),Q_{c,\mathrm{peak}}^{\pi,P}(s,\cdot)\rangle. \label{eq:q_value}
\]
For \(\gamma<1\), the operator induced by \eqref{eq:max_value} is a contraction, and hence \(V_{c,\mathrm{peak}}^\pi\) can be computed for a given policy \(\pi\) by fixed-point iteration. For finite-horizon episodic setup, we can simply take $\gamma=1$.  Inspired from \cite{yu2022reachability}, we then consider the  modified optimization problem:
\begin{align}\label{eq:rcrl}
\min_{\pi}\; J^P_r(\pi)
\qquad
\text{s.t.}\qquad
\max_{s \in S_0}\!\left[V_{c,\mathrm{peak}}^{\pi,P}(s)\right]\le b.
\end{align}
where $S_0$ is the set of initial states. 
This is equivalent to the peak-cost constraint defined above for the deterministic case. 
\vspace{-0.03in}
\begin{remark}
    Note that for notational simplicity we define the peak cost as a state-based cost, however, it can be extended to state-action-based cost easily. 
\end{remark}
In \cite{yu2022reachability}, this is solved by  Lagrangian-based objective:
\begin{align}
\min_{\pi}\max_{\beta\ge 0}\;
\!\left[V_r^{\pi,P}(s)+\beta(s)V_{c,\mathrm{peak}}^{\pi,P}(s)\right],
\label{eq:lagrangian}
\end{align}
where \(\beta(s)\) is a state-dependent Lagrange multiplier used to penalize constraint violations. 

\subsection{Limitations of the Lagrangian Approach in \cite{yu2022reachability}}
While \cite{yu2022reachability} adopts the Lagrangian formulation in \eqref{eq:lagrangian}, it does not establish a strong-duality guarantee. This is a critical issue because the peak-cost constrained problem is fundamentally different from the standard CMDP: the constraint is no longer based on an additive discounted cumulative cost, but instead depends on the maximum cost encountered along a trajectory. In particular, the reachability set $\mathcal{S}^{\pi,P}$ inherently depends on the policy $\pi$.

In this paper, we show that the peak-cost constrained problem does \emph{not} necessarily admit zero duality gap. To the best of our knowledge, this is the first result showing that, unlike the standard CMDP, the peak-constrained problem may fail to satisfy strong duality. The key reason is that the problem is no longer convex in the state-action occupancy measure, due to its dependence on the trajectory-wise maximum cost.

\begin{theorem}\label{thm:non-zero-duality}[informal]
The problem stated in (\ref{eq:rcrl}) may have a non-zero duality gap even if the Slater's condition is satisfied.
\end{theorem}
The proof of Theorem~\ref{thm:non-zero-duality} is in Appendix~\ref{sec:proof}.

\begin{remark}[Implication for primal-dual methods]
Theorem~\ref{thm:non-zero-duality} shows that the modified hard-max constraint can destroy the strong-duality structure that underlies standard primal-dual methods for CMDPs. In ordinary CMDPs, the objective and constraints are linear in the discounted state-action occupancy measure, which yields a convex primal problem and zero duality gap under mild conditions. By contrast, the modified hard-max constraint induces a nonlinear, nonstandard feasible set, and the associated Lagrangian relaxation can be strictly loose. Consequently, a primal-dual method applied directly to this modified constraint is not, in general, guaranteed to recover the true constrained optimum. This observation motivates the use of alternative surrogate formulations and direct policy-space optimization methods for such peak-style safety objectives.
\end{remark}
\vspace{-0.03in}
\subsection{Limitations and a Robust Formulation}

Problem~(\ref{eq:rcrl}) does not account for the mismatch between the simulated environment used for training and the real environment in which the learned policy is ultimately deployed. In practice, policies are often trained in simulation and then transferred to the real world, where the transition dynamics may differ from those of the simulator. As a result, a policy that appears feasible in simulation may violate the safety constraints after deployment. This motivates a robust formulation in which the learned policy must remain feasible under the worst-case model within a prescribed uncertainty set. Accordingly, we consider the following robust counterpart:
\begin{align}\label{eq:robust_rcrl}
\min_{\pi}\max_{P\in \mathcal{P}}\; J_r^{P}(\pi)
\qquad
\text{s.t. }
\max_{P\in \mathcal{P}}\max_{s\in S_0}\!\left[V_{c,\mathrm{peak}}^{\pi,P}(s)\right]\le b.
\end{align}
This formulation seeks to maximize the worst-case expected cumulative reward while ensuring that the worst-case peak constraint remains below the prescribed threshold.

For the uncertainty set, one natural choice is the rectangular uncertainty set
\[
\mathcal{P}_{(s,a)}
=
\left\{
P \in \Delta(S) : D\!\left(P, P_0(\cdot \mid s,a)\right) \leq \rho
\right\},
\]
where \(D(\cdot,\cdot)\) measures the distance between two probability distributions, and $P_0$ is the nominal model encountered in simulator. Rectangular uncertainty sets are particularly attractive because they often lead to tractable robust dynamic programming formulations, whereas evaluating the robust value function under non-rectangular uncertainty sets is, in general, NP-hard \cite{iyengar2005robust}.

For continuous state spaces, one may instead consider uncertainty sets defined through an integral probability metric (IPM)~\cite{zhou2023natural}, which admits an efficient representation of the worst-case value function under linear function approximation. In particular, the IPM between two distributions \(p\) and \(q\) is defined as
\[
d_F(p,q)
=
\sup_{f\in F}\left\{p^\top f - q^\top f\right\},
\qquad
F \subseteq \mathbb{R}^{|S|}.
\]
Under suitable choices of \(F\), the IPM recovers important distributional distances; for example, under special conditions it reduces to the total variation distance~\cite{muller1997integral}. Moreover, IPM-based uncertainty sets yield a closed-form expression for the robust Bellman operator under linear approximation for value function (i.e., $V_{c,\mathrm{peak}}^{\pi}(s)=\Psi(s)^Tw^{\pi,c}$)
\begin{equation}
(1-\gamma)c(s,a) + \gamma \max\{c(s,a)+V_{c,\mathrm{peak},w}^{\pi}\} + \rho \|w_{c,2:d}\|_2.
\label{eq:IPM}
\end{equation}
where $w_{c,2:d}$ is the weight value for critic function and we exclude $w_{c,1}$ because we assume a normalized weight. 
\vspace{-0.04in}
\section{Solution Methodology}
In this section, we discuss our problem to solve peak-cost constrained robust MDP as described in (\ref{eq:robust_rcrl}).

\subsection{Robust Peak-Cost Constrained MDP Formulation}
Note that since the problem in (\ref{eq:rcrl}) may not have a zero duality gap we have to consider a surrogate problem. Inspired from \cite{gangulyefficient}, we consider the following surrogate problem for $\beta>0$. 
\begin{equation}
   \min_{\pi} \max\{\max_{P} J_r^{\pi,P}/\beta, (\max_{P,s\in S_0} V_{c,\mathrm{peak}}^{\pi,P}(s)-b)\}\label{eq:main_obj}
\end{equation}
\cite{gangulyefficient} adopted this surrogate problem for the robust CMDP problem, here, we apply for the robust peak cost constrained problem. Note that when the policy $\pi$ is infeasible and $\beta$ is large enough, then, it would try to optimize the constraint value function to make it satisfy the constraint. On the other hand, when the policy is feasible then it would optimize the objective if the reward is positive if the peak-cost constraint is satisfied. 

  The effectiveness of the surrogate problem in (\ref{eq:main_obj}) is shown in the following proposition. 
  \begin{proposition}\label{prop:equivalence}

Let $\hat{\pi}^*$ be an optimal solution of (\ref{eq:main_obj}), then $J_r^{\hat{\pi}^*}\leq J_r^{\pi^*}$ where $\pi^*$ is an optimal solution for (\ref{eq:robust_rcrl}). Further if $\beta\geq (1-\gamma)^{-1}/\epsilon$, then $(\max_{P,s\in \mathcal{S}^{\pi,P}}V_{c,\mathrm{peak}}^{\pi,P}(s)-b)_+\leq \epsilon$.
  \end{proposition}
  The above proof is adapted from Proposition 1 in \cite{gangulyefficient}, and has been shown in Appendix~\ref{sec:proof}. \cite{gangulyefficient} also provides the convergence rate guarantee for a robust natural policy gradient which solves (\ref{eq:main_obj}) for the finite state-space and finite action space. Our goal is to apply for the large state-space and continuous action space problem, hence, we do not consider the natural policy gradient. We rather consider a robust actor-critic-based approach, and we use PPO-method \cite{schulman2017proximal} for the actor. 

\section{Proposed Algorithm}
In this section, we present our approach for solving the robust peak-constrained MDP problem formulated in \eqref{eq:main_obj}. Our algorithm is based on the Actor-Critic framework modified for the peak-constrained, and the robust value function. 

\subsection{Actor--Critic Optimization}
To account for robustness, we adapt the temporal-difference (TD) errors for both the reward critic and the peak-cost critic using the IPM-based robustification in \eqref{eq:IPM}. Specifically, we consider the following advantage function
\begin{align}
\delta_{t,r}
&=
-r_t + \gamma V_r(s_{t+1}) - V_r(s_t) + \rho_r \|w_{r,2:d}\|_2,
\label{eq:td_error_reward}
\\
\delta_{t,c}
&=
\big(1-\gamma \big)c_t+ \gamma \max\!\big(c_t,\, V_{c,\mathrm{peak}}(s_{t+1})\big) \nonumber \\ & \qquad \qquad \qquad - V_{c,\mathrm{peak}}(s_t) + \rho_c \|w_{c,2:d}\|_2.
\label{eq:td_error_cost}
\end{align}
Note that the added regularization is due to the robust value function metric as discussed in (\ref{eq:IPM}).
For smooth approximation, instead of max operator, we use Log-Sum-Exponential(LSE).


Using these TD errors, we construct the generalized advantage estimates (GAEs) for the reward objective and 1-step advantage estimate for peak-cost objective as
\begin{align}
A_{t,r}
&=
\sum_{l=0}^{\infty} (\gamma \lambda)^l \delta_{t+l,r},
\label{eq:gae_reward}
\\
A_{t,c}
&=
\delta_{t,c}.
\label{eq:gae_cost}
\end{align}

Here $\lambda$ is the GAE decay. The corresponding value targets for the reward and peak-cost critics are given by
\begin{align}
v_{\text{target},r}(s_t)
&=
A_{t,r} + V_r(s_t),
\label{eq:targetv_reward}
\\
v_{\text{target},c}(s_t)
&=
A_{t,c} + V_{c,\mathrm{peak}}(s_t)
\label{eq:targetv_cost}
\end{align}

The actor is updated using the PPO-style clipped objective \cite{schulman2017proximal} depending on whether we are optimizing the objective or the constraint
\begin{align}
L_{\text{actor}}
=
\mathbb{E}_t \!\left[
\min\!\left(
r_t(\theta)\, A_{t,j},\;
\mathrm{clip}\!\big(r_t(\theta),\, 1-\epsilon,\, 1+\epsilon\big)\, A_{t,j}
\right)
\right],
\label{eq:actor_loss}
\end{align}
where \(r_t(\theta)\) denotes the policy ratio and \(A_{t,j}\) is the advantage signal used for the policy update, for $j=r,c$.

The robust critic losses for the reward and peak-cost value functions are defined as
\begin{align}
L_{V_r}
&=
\Big(V_r(s_t) - v_{\text{target},r}(s_t)\Big)^2,
\label{eq:loss_reward}
\\
L_{V_{c,\mathrm{peak}}}
&=
\Big(V_{c,\mathrm{peak}}(s_t) - v_{\text{target},c}(s_t)\Big)^2.
\label{eq:loss_cost}
\end{align}

Algorithm~\ref{algo:main_algo} summarizes the complete training procedure for the robust peak-constrained MDP.
\begin{algorithm}
\caption{\ours}
\label{algo:main_algo}
\KwIn{Initial policy parameters $\theta_0$, initial value function parameters $\phi_{r,0}$ and $\phi_{c,0}$}
\KwParam{Discount factor $\gamma$, Peak-Cost budget $b$, GAE decay $\lambda$, Lagrangian multiplier $\beta$, PPO clipping parameter $\epsilon$, learning rates $\alpha_{\text{actor}}$, $\alpha_{\text{critic,r}}$, $\alpha_{\text{critic,c}}$ Reward/Cost Selector $\psi(.)$}

\textbf{Initialize:} Policy $\pi_\theta$ with parameters $\theta_0$, robust value functions $V_r$ and $V_{c,\mathrm{peak}}$ with parameters $\phi_{r,0}$ and $\phi_{c,0}$\;

\For{$t = 0, 1, \dots, T-1$}
 {
    Collect a set of trajectories $\mathcal{D}_t = \{\tau_i\}_{i=1}^N$ by running policy $\pi_\theta$ in the environment\;

    $A = \frac{1}{|D_t|} \sum_{s \in D_t} \frac{V_r(s)}{\beta}$\;
    $B = \max_{s \in D_t} \left(V_{c,\mathrm{peak}}(s) - b\right)$\;
    $F = \arg\max \{A, B\}$\;

    Estimate Advantage \& Update Critic: $\psi(F)$\;
    
    Update Policy using Eqn~\ref{eq:actor_loss} and update $\theta_{t}$ using SGD with learning rate $\alpha_{\text{actor}}$\;
 }
 \textbf{Output:} Optimized policy $\pi_\theta$
\end{algorithm}

\begin{algorithm}
\caption{Reward/Cost Selector $\psi$}
\label{algo:update_reward_or_cost}
\KwIn{Choice function F, Discount factor $\gamma$, GAE decay $\lambda$,PPO clipping parameter $\epsilon$, learning rates $\alpha_{\text{actor}}$, $\alpha_{\text{critic,r}}$, $\alpha_{\text{critic,c}}$}


\uIf{$F = 0$}{
    \textbf{Update Reward Critic:}\;
    Compute TD errors $\delta_{t,r}$ using Eqn~\ref{eq:td_error_reward}\;    
    Estimate advantages $A_{t,r}$ using Eqn~\ref{eq:gae_reward}\;
}
\uElseIf{$F = 1$}{
    \textbf{Update Cost Critic:}\;
    Compute TD errors $\delta_{t,c}$ using Eqn~\ref{eq:td_error_cost}\;    
    Estimate advantages $A_{t,c}$ using Eqn~\ref{eq:gae_cost}\;
}
Estimate targets $v_{target,r}$ using Eqn~\ref{eq:targetv_reward}\;
Estimate reward critic loss $L_{V_r}$ using Eqn~\ref{eq:loss_reward}\;
Update critic parameters $\phi_{r}$ using SGD with learning rate $\alpha_{\text{critic,r}}$\;

Estimate targets $v_{target,c}$ using Eqn~\ref{eq:targetv_cost}\;
Estimate cost critic loss $L_{V_{c,\mathrm{peak}}}$ using Eqn~\ref{eq:loss_cost}\;
Update critic parameters $\phi_{c}$ using SGD with learning rate $\alpha_{\text{critic,c}}$\;
\end{algorithm}
 Note that we take a sample average of the reward value function scaled by $1/\beta$ ($A$), if it is higher than the maximum value function corresponding to the state sampled from the trajectory, then we update the actor based on the average reward value function, otherwise we update the actor based on the constraint value function to minimize the constraint violation.

\vspace{-0.05in}
\section{Experiments}
\vspace{-0.03in}
\subsection{Gym Environment: Constrained CartPole}

\subsubsection{Environment and Training Details}

\parab{Constrained CartPole Environment}
The CartPole environment is a widely used reinforcement learning benchmark introduced in OpenAI Gym~\cite{brockman2016openai}.

Unlike the standard discrete-action CartPole environment, we use a continuous action space where the agent applies a force:
$
F \in [-10, 10].
$
The constrained environment enforces safety by requiring the cart position to remain within the safe region:
$
x \in [-1, 1].
$
The standard termination limits are:
$
x \in [-2.4, 2.4],
\qquad
\theta \in [-12^\circ, 12^\circ].
$
The cost function is:
$
c(s,a) =
\begin{cases}
|x|, & \text{if } |x| > 1.0, \\
10, & \text{if early termination before 450 steps}.
\end{cases}
$
The objective is to maximize cumulative reward while keeping the peak cost below a predefined baseline. The maximum episode length is 500 steps.

To model transition uncertainty, we perturb gravity during training as:
$
g' = g + \mathcal{N}(0, 0.5).
$
This stochastic perturbation requires the policy to remain robust.




\subsubsection{Experiments and Analysis}


\begin{enumerate}
\item{\textbf{Effectiveness of our Surrogate Objective }}
\label{exp:pd_ours}
\begin{figure}
    \centering
    \begin{subfigure}{0.22\textwidth}
        \centering
        \includegraphics[width=\linewidth]{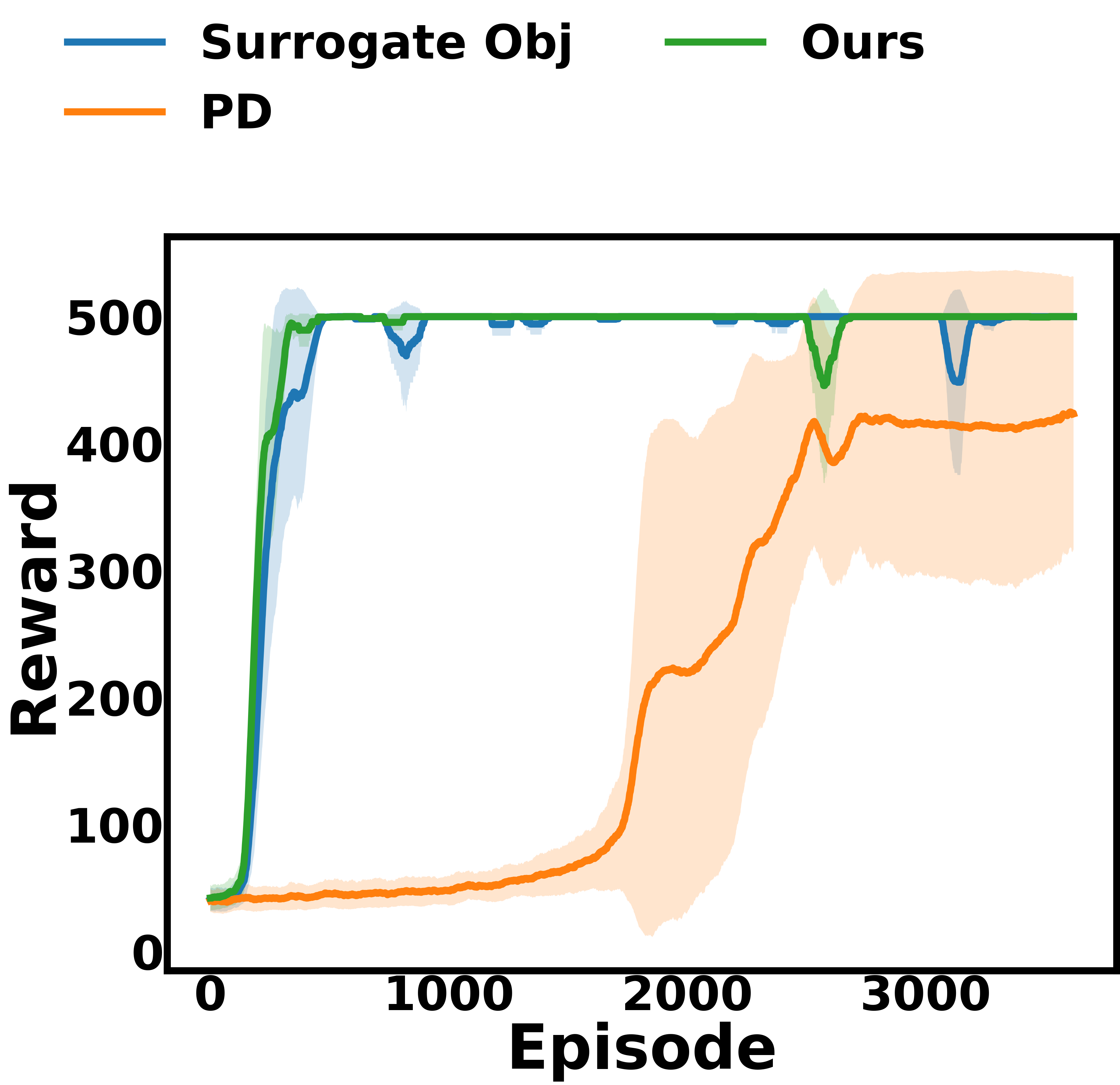}
        \caption{Returns}
    \end{subfigure}
    \begin{subfigure}{0.22\textwidth}
        \centering
        \includegraphics[width=\linewidth]{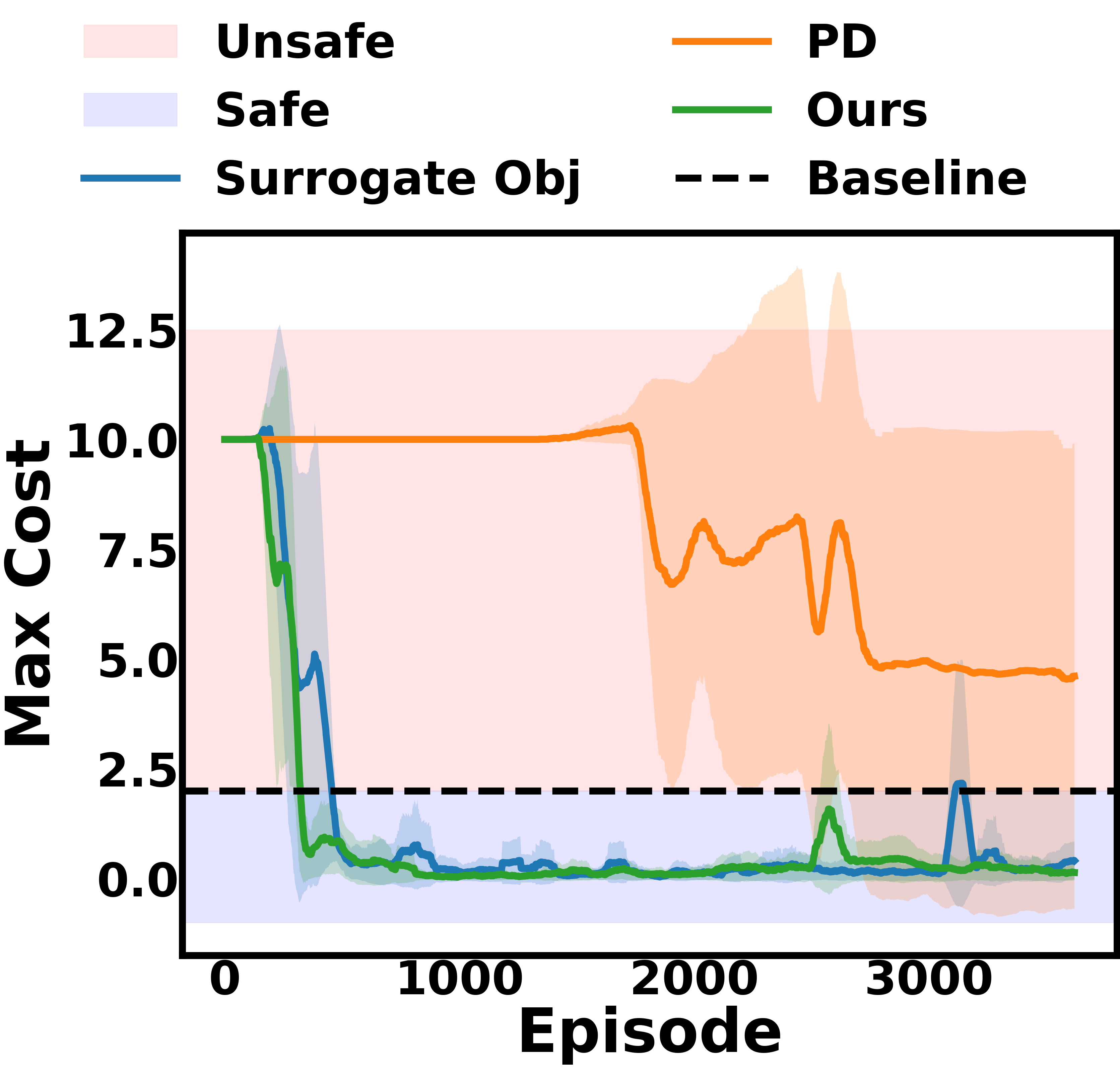}
        \caption{Peak-Cost}
    \end{subfigure}
    \caption{\small Comparison of Primal Dual(PD) and Surrogate Objective trained in non-perturbed environment, and \ours\ trained in perturbed CartPole environment. Here we want to maximize return while restrict peak-cost below baseline. \ours\ has learned peak-cost constraint and has higher rewards than PD.}
    \label{fig:pd_ours_train}
\end{figure}

Fig~\ref{fig:pd_ours_train} compares the training of Primal-Dual (PD), Surrogate Objective in \eqref{eq:main_obj} on a non-perturbed environment and \ours(Surrogate Objective trained on a perturbed environment). The surrogate objective based methods successfully enforce the peak-cost constraint and achieve higher returns compared to PD. PD learns lower returns and has higher peak-cost than the baseline because it learns a sub-optimal policy.

\item{\textbf{Inference in Perturbed Environment }}
\label{exp:inference}

\begin{figure}
    \centering
    \includegraphics[width=0.48\textwidth]{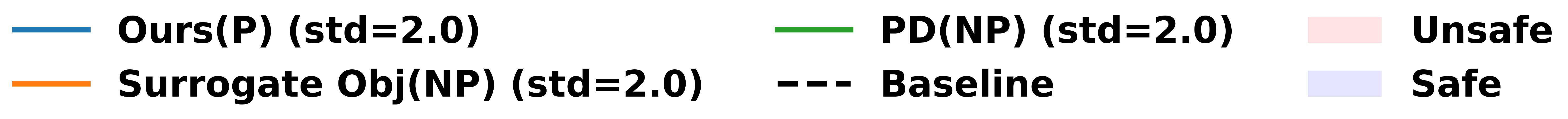} 
    \vspace{0.5cm} 
    \begin{subfigure}{0.23\textwidth}
        \centering
        \includegraphics[width=\linewidth]{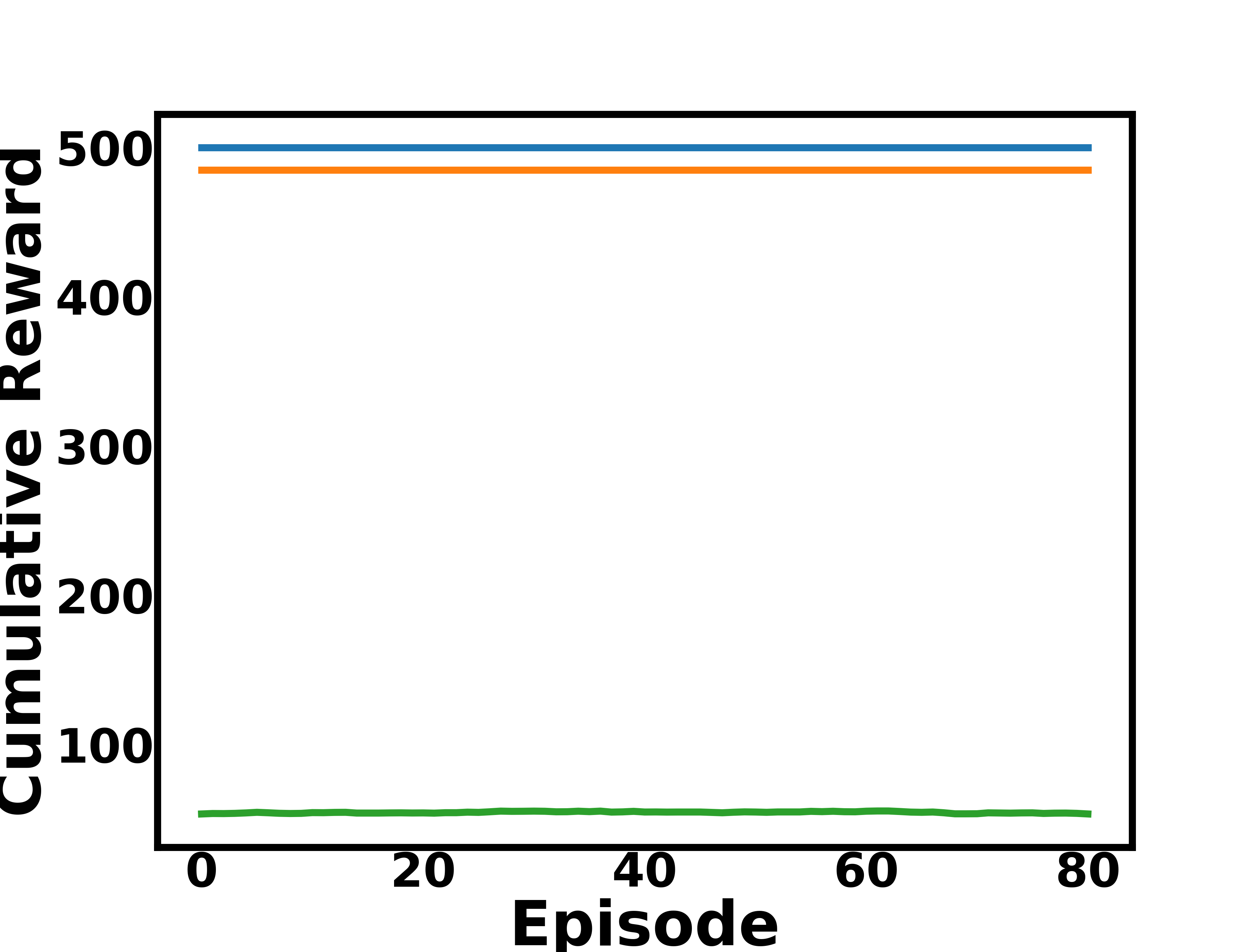}
        \caption{Returns}
    \end{subfigure}
    \begin{subfigure}{0.23\textwidth}
        \centering
        \includegraphics[width=\linewidth]{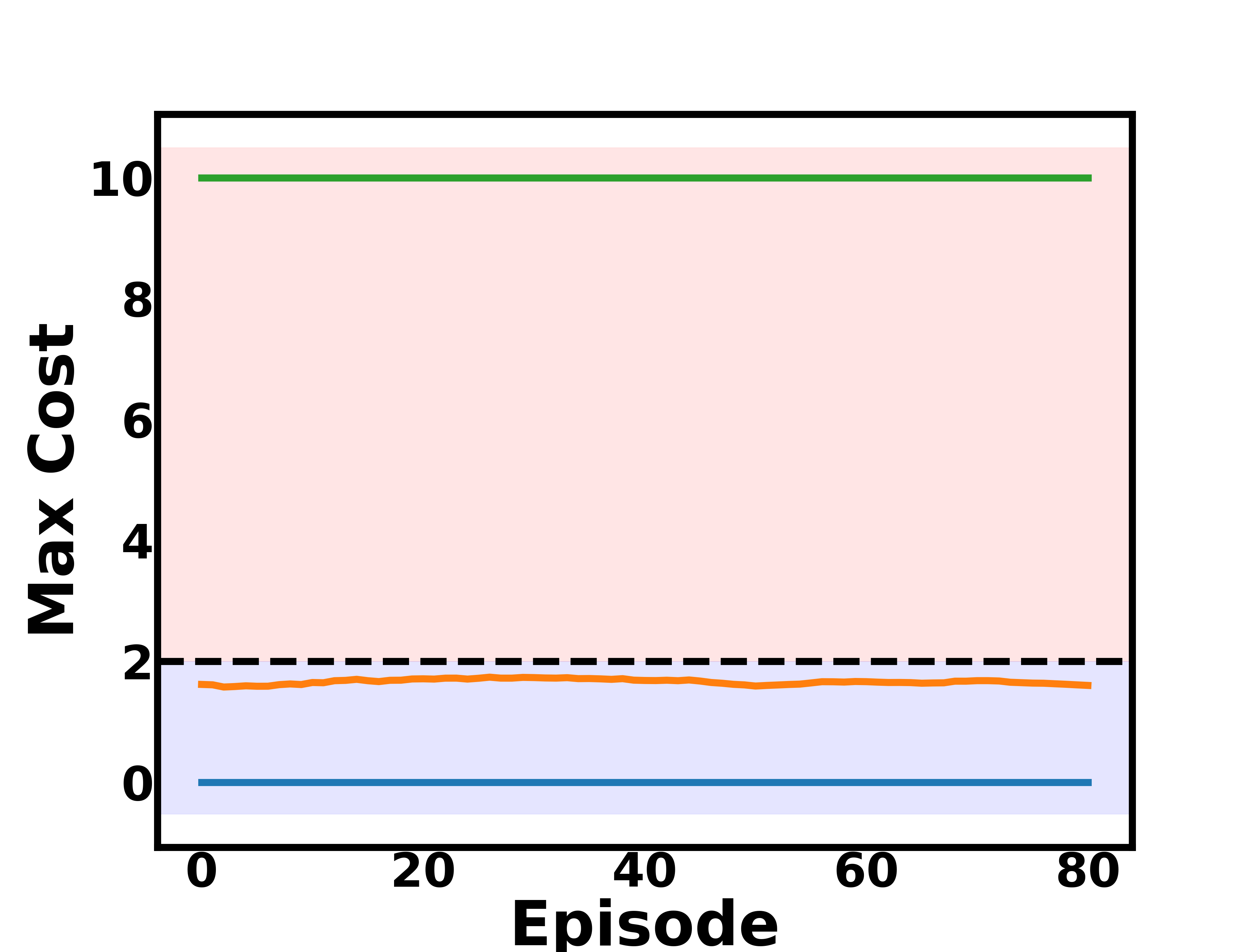}
        \caption{Peak-Cost}
    \end{subfigure}
    \vspace{-0.1in}
    \caption{\small Comparison of inference performances of PD trained on non-perturbed(NP) environment, Surrogate Objective~\eqref{eq:main_obj} trained on NP environment and \ours\ trained in perturbed (P) environment.\ours\ performs best with a peak-cost of zero. PD fails to adapt to a perturbed environment at inference.}
    \label{fig:inference}
\end{figure}

The inference is done with a perturbation of $g' = g + \mathcal{N}(0,2.0)$. The comparison in Fig~\ref{fig:inference} highlights the superior performance of \ours\ trained in a perturbed (P) environment, achieving the best results with a peak-cost of zero. In contrast, Surrogate Objective, trained in non-perturbed (NP) environments, has a peak-cost below baseline and PD fails to meet the peak-cost constraint and adapt to the robustness required during inference, evident from its rewards as well. 
\end{enumerate}

\subsubsection{Training across Different Cost Functions}
We study \ours\ under two different CartPole cost functions to evaluate the generalizability of our approach. Cost1 (C1) penalizes cart-position deviations beyond $1.0$ from the origin and adds an early termination penalty of $10.0$ when the cart position exceeds $2.4$ or the pole angle exceeds $12^\circ$ before 450 steps, with baseline $b=2$. Cost2 (C2) penalizes deviations beyond the safe cart-position limit and pole-angle deviations greater than $8^\circ$, with costs increasing near the termination thresholds, and uses baseline $b=0.5$. Fig.~\ref{fig:c2} shows that \ours\ can effectively learn under different cost definitions, demonstrating robustness and adaptability across safety specifications.
\begin{figure}
    \centering
    \begin{subfigure}{0.21\textwidth}
        \centering
        \includegraphics[width=\linewidth]{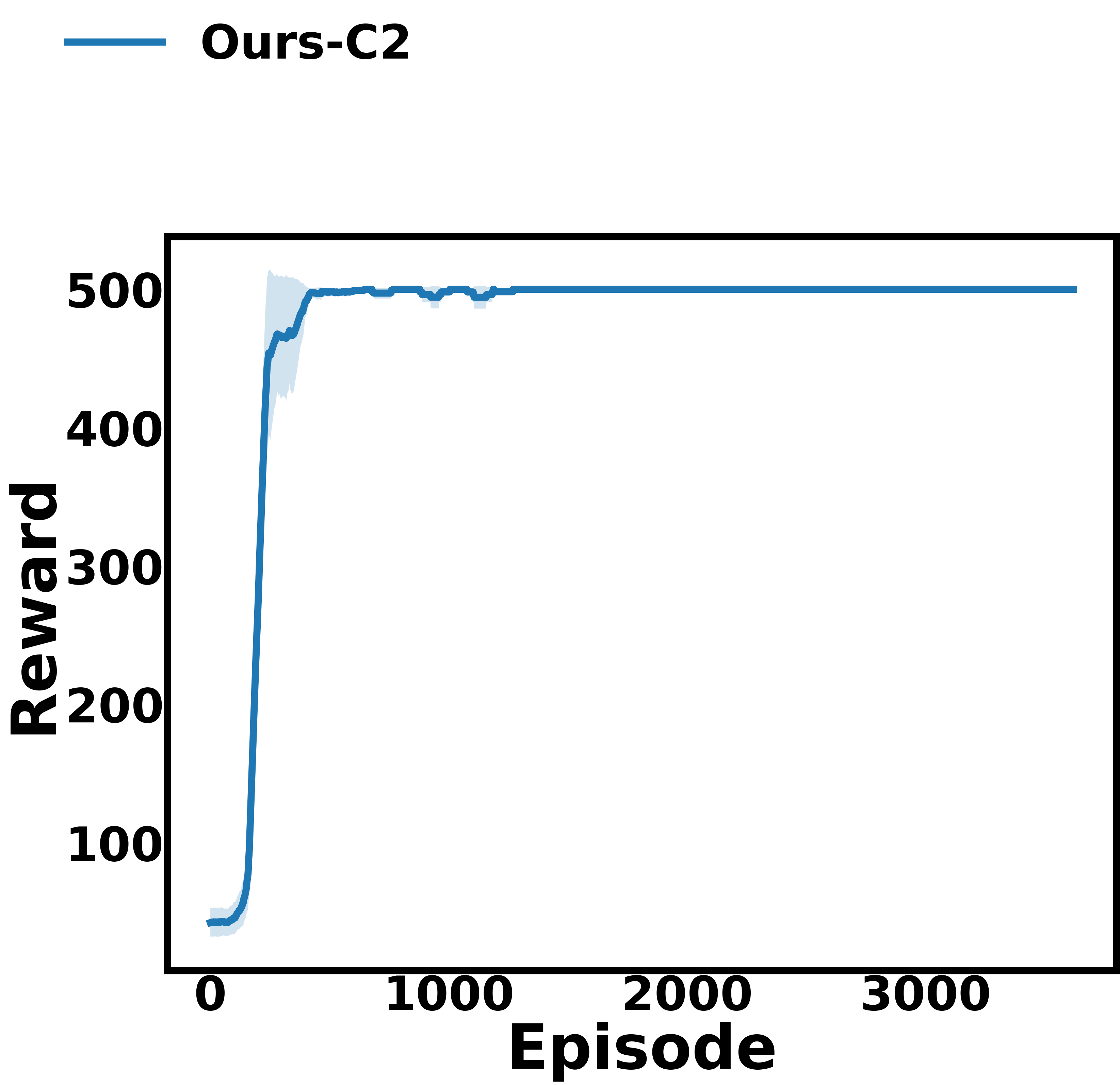}
        \caption{Returns}
    \end{subfigure}
    \begin{subfigure}{0.21\textwidth}
        \centering
        \includegraphics[width=\linewidth]{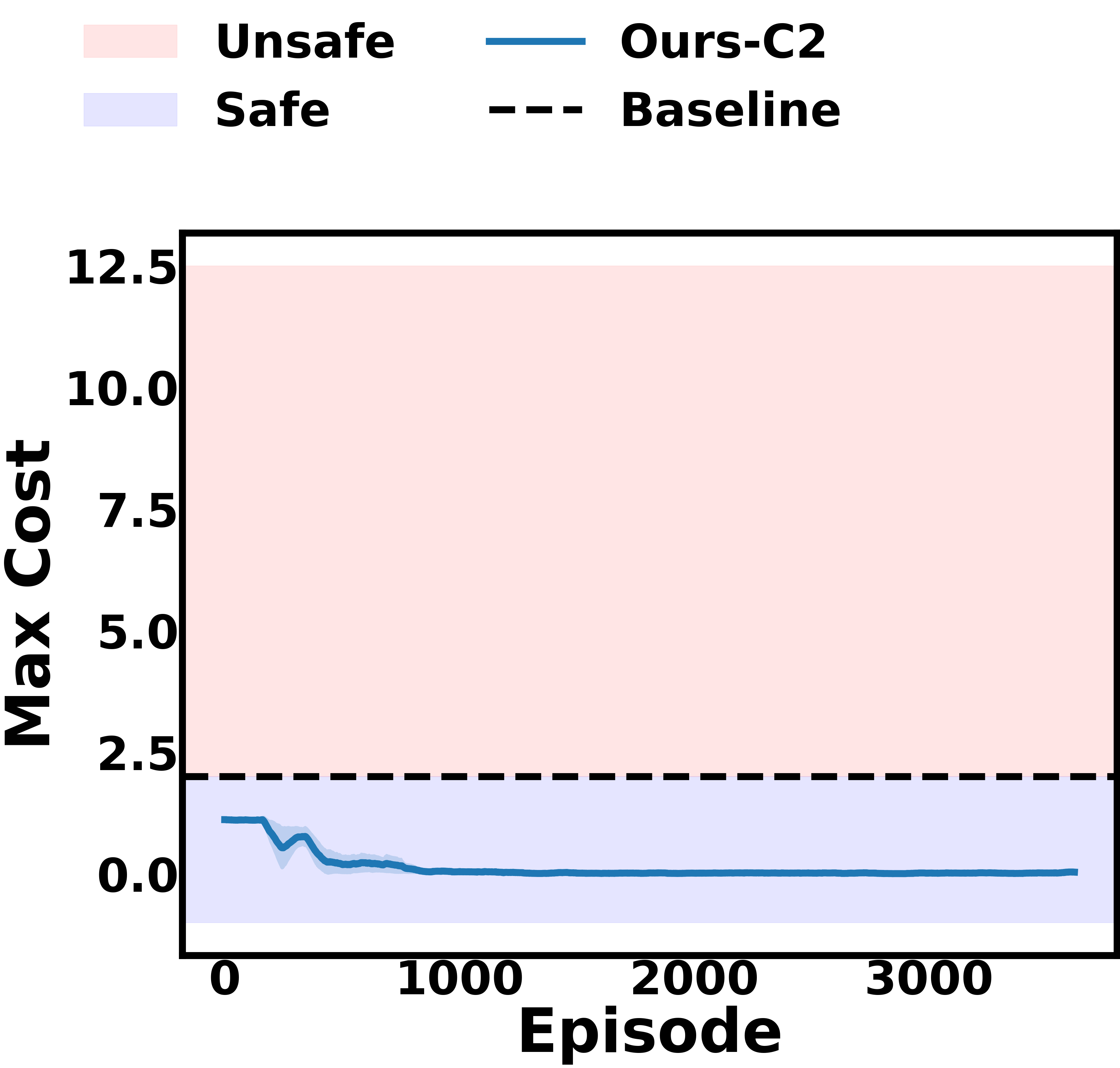}
        \caption{Peak-Cost}
    \end{subfigure}
    \caption{\small Learning \ours\ on a different cost C2 constraining both angle and position. The RP-CRL policy learns the peak-cost below baseline in C2 as well.}
    \label{fig:c2}
\end{figure}

\subsubsection{Training across Different Levels of Perturbations}

We perform a study of \ours\ learning two different perturbation levels in the perturbed CartPole environment. This experiment is conducted to investigate how well \ours\ learns when perturbations in the environment are different. Fig~\ref{fig:perturbation_train} shows the training curves for different perturbation levels. We perturb the gravity $g' = g + \mathcal{N}(0,\sigma)$. For comparison, we consider $\sigma =  \{0.5,2\}$. The learning curves demonstrate \ours\ learns different perturbation levels effectively.

\begin{figure}
    \centering
    \begin{subfigure}{0.22\textwidth}
        \centering
        \includegraphics[width=\linewidth]{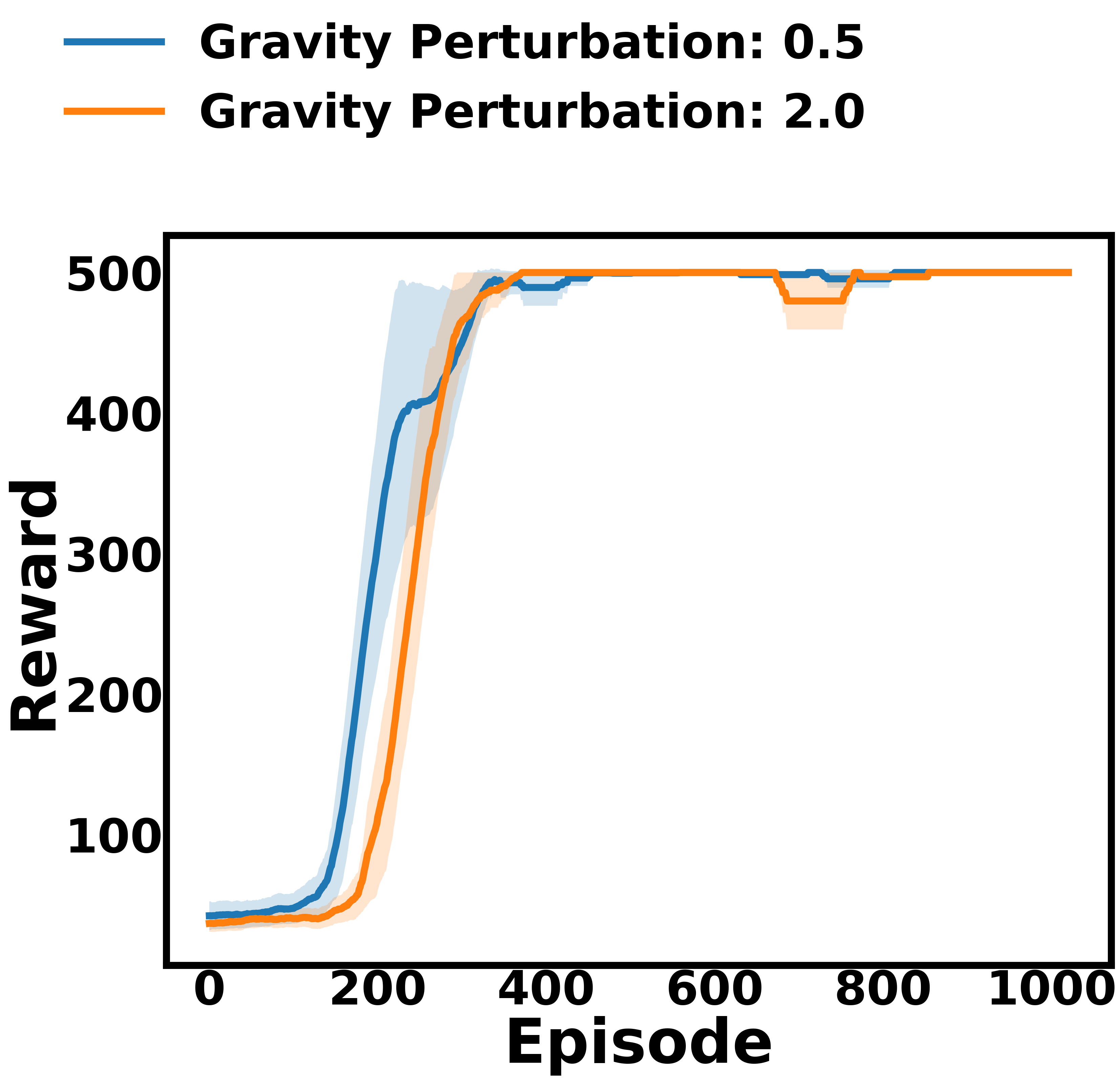}
        \caption{Returns}
    \end{subfigure}
    \begin{subfigure}{0.25\textwidth}
        \centering
        \includegraphics[width=\linewidth]{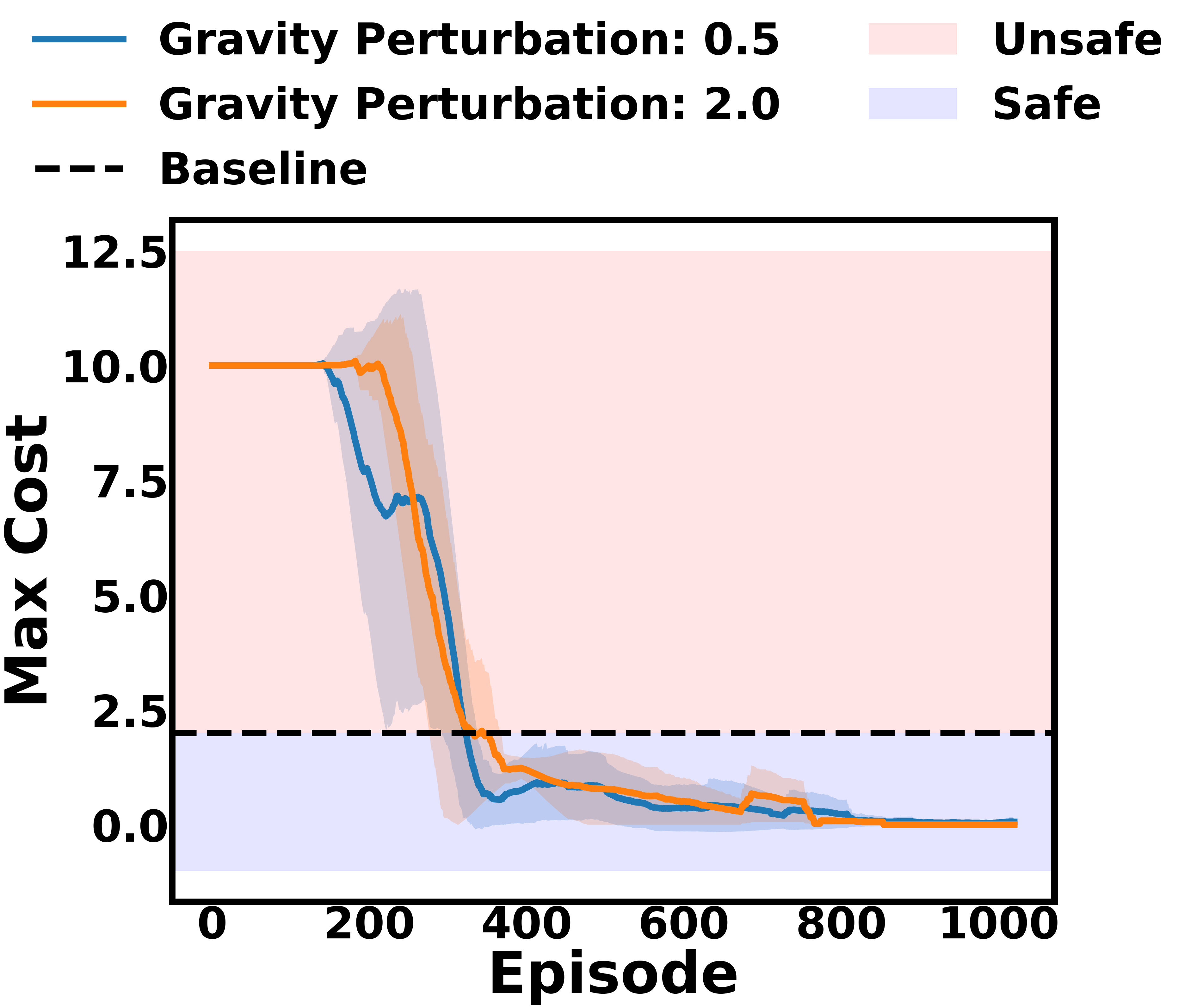}
        \caption{Peak-Cost}
    \end{subfigure}
    \caption{\small Comparison of training curves of \ours\ for different perturbation levels in the environment. The perturbation is introduced in gravity. \ours\ learns maximized returns and bounded peak-cost below baseline for both perturbation levels.}
    \label{fig:perturbation_train}
\end{figure}

\vspace{-0.04in}
\subsection{MuJoCo Environment: Constrained HalfCheetah}

\subsubsection{Environment and Training Details}

\parab{Constrained HalfCheetah Environment}
We also evaluate our method on a constrained HalfCheetah environment based on the MuJoCo continuous-control benchmark. 

We introduce the constraint based on maximum absolute action torque. Let $a_t$ denote the action at time $t$. The instantaneous cost is defined as:
$$
c_t =
\max \left(
\max_i |a_{t,i}| - 0.5,\ 0
\right),
$$
where $a_{t,i}$ is the torque applied at the $i$-th joint. Thus, the agent incurs zero cost when all absolute torque values are below the threshold $0.5$, and incurs a positive cost only when the maximum torque exceeds this threshold.


To evaluate robustness, we train \ours\ in a perturbed HalfCheetah environment by perturbing gravity during training. Specifically, the gravity is sampled as:
$
g' = g + \mathcal{N}(0, \sigma),
$
where $\sigma$ controls the perturbation level. For HalfCheetah, we use gravity perturbations of $2.0$ during training. At inference, we evaluate the learned policies under different perturbed dynamics to test robustness and safety generalization.




\subsubsection{Experiments and Analysis}

We evaluate \ours\ on the constrained HalfCheetah environment and compare it with the same baselines: Primal Dual method for Cumulative Cost, RCRL~\cite{yu2022reachability} for Peak-Cost Constraint,  and Surrogate Objective trained in the Non-Perturbed environment. The objective is to test whether training with perturbed dynamics improves robustness and persistent safety under uncertain dynamics.

\begin{figure}
    \centering
    \begin{subfigure}{0.20\textwidth}
        \centering
        \includegraphics[width=\linewidth]{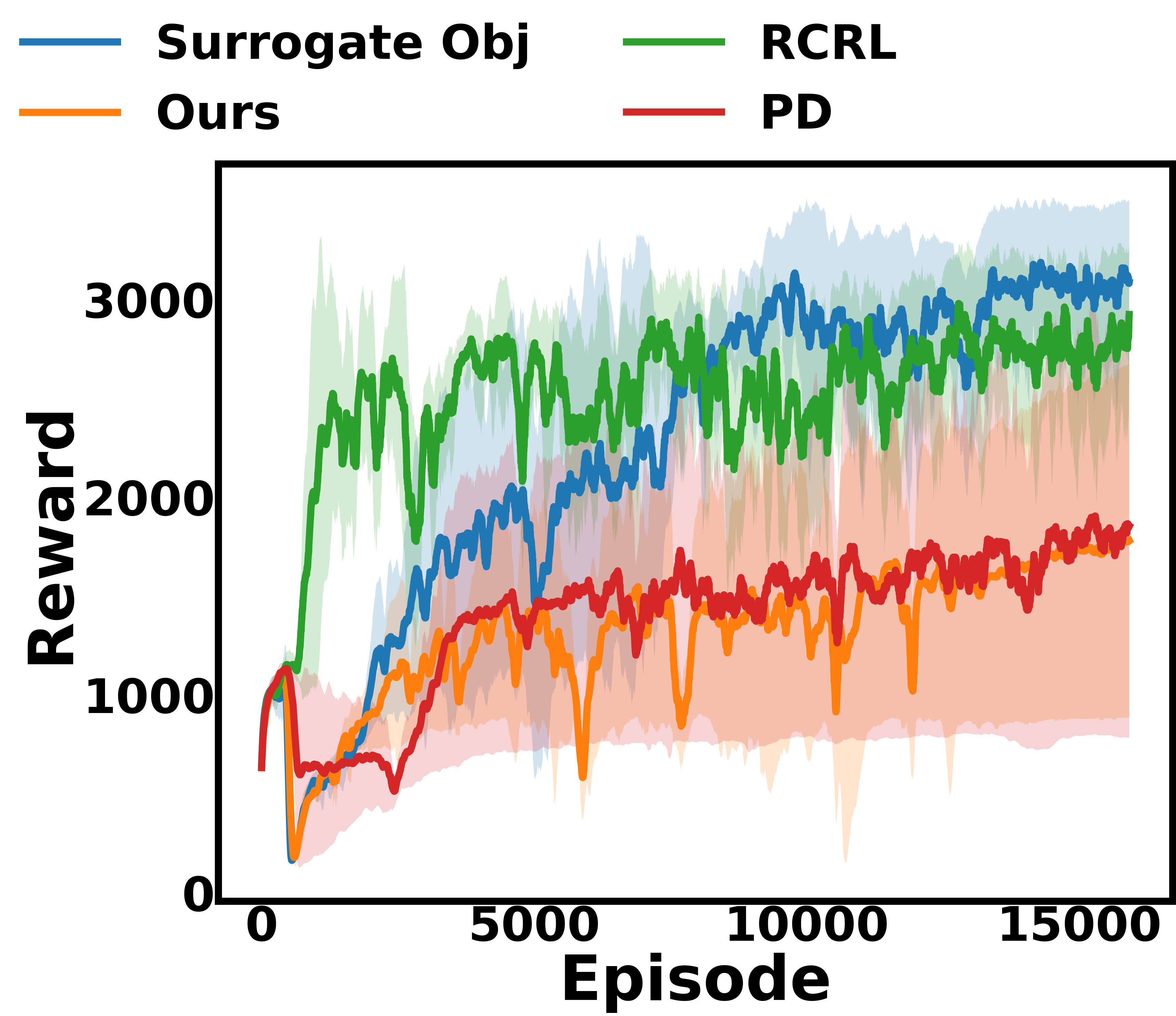}
        \caption{Returns}
    \end{subfigure}
    \begin{subfigure}{0.24\textwidth}
        \centering
        \includegraphics[width=\linewidth]{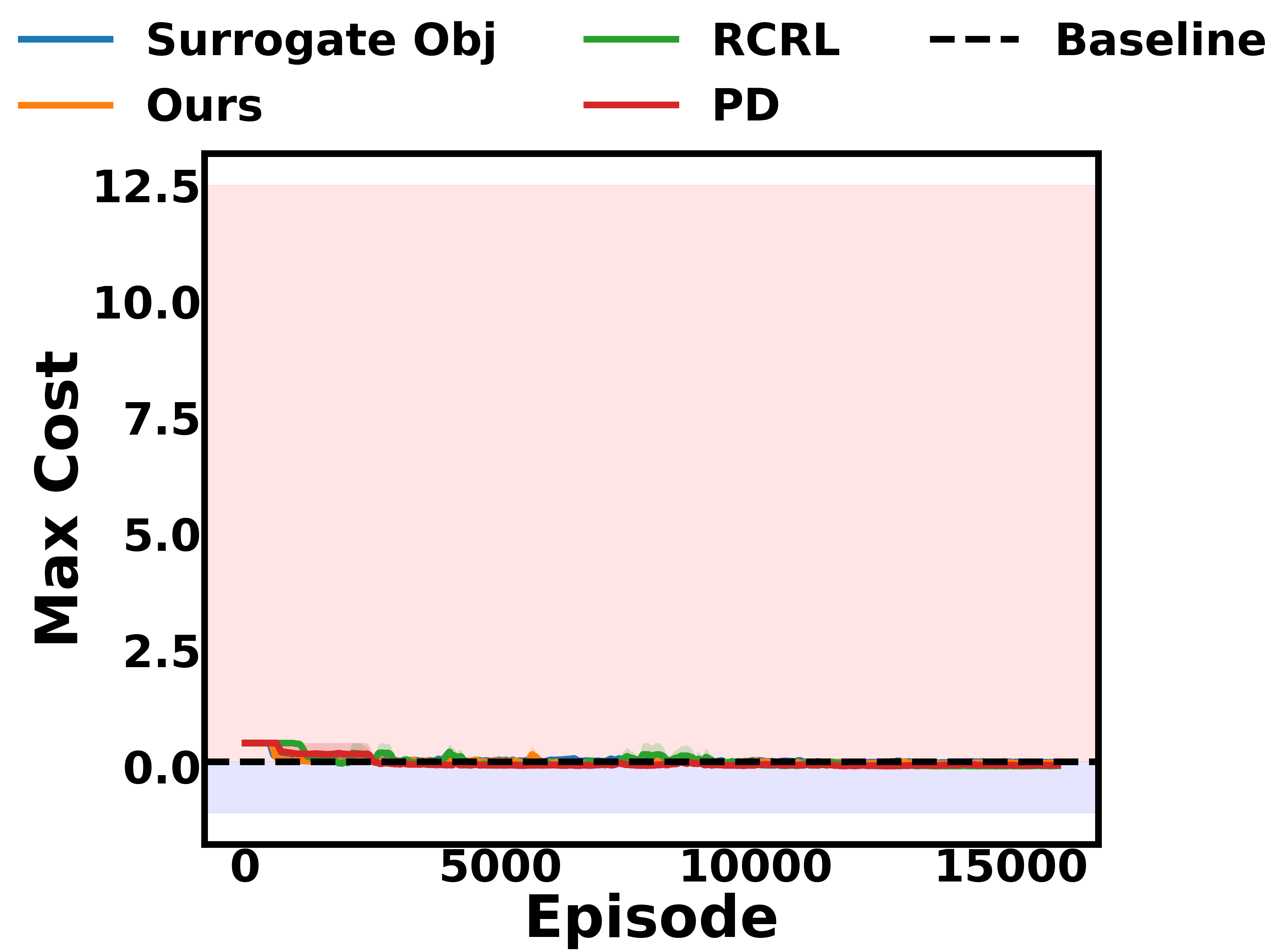}
        \caption{Peak-Cost}
    \end{subfigure}
    \caption{\small Comparison of Primal Dual(PD), RCRL and Surrogate Objective(SO) trained in non-perturbed environment, with \ours\ trained in perturbed HalfCheetah environment. \ours\ has learned peak-cost constraint and has similar rewards than PD. Non-perturbed SO and RCRL yield higher returns while training}
    \label{fig:hc_train}
\end{figure}


\begin{figure}
    \centering
    \includegraphics[width=0.48\textwidth]{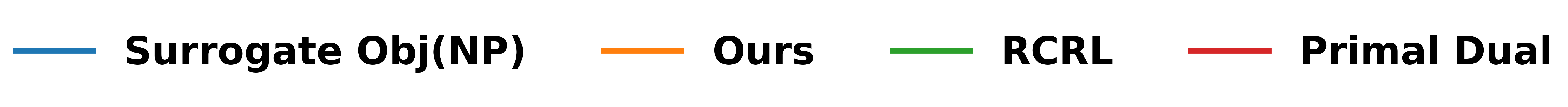} 
    \vspace{0.5cm} 
    \begin{subfigure}{0.23\textwidth}
        \centering
        \includegraphics[width=\linewidth]{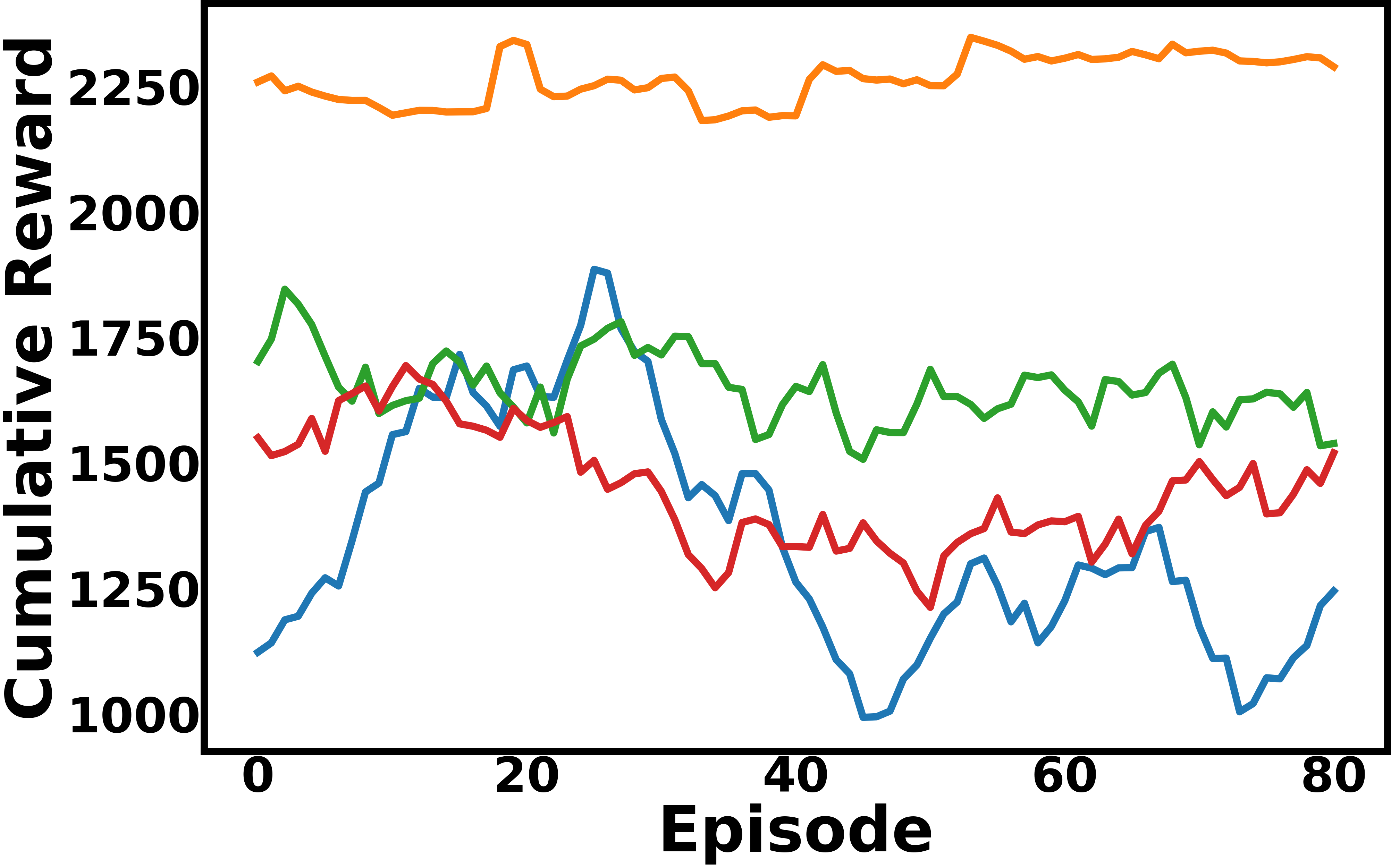}
        \caption{Returns}
    \end{subfigure}
    \begin{subfigure}{0.23\textwidth}
        \centering
        \includegraphics[width=\linewidth]{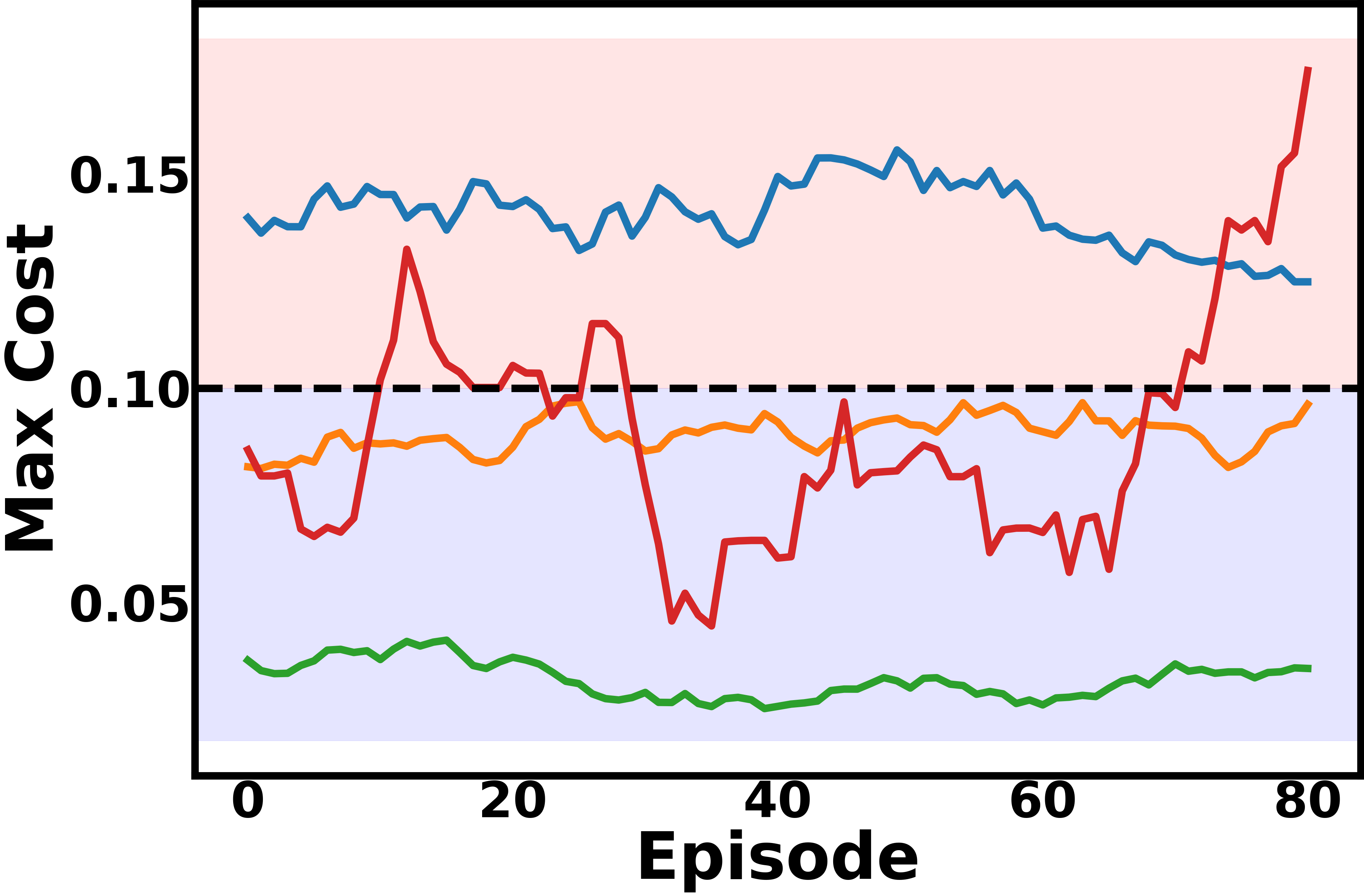}
        \caption{Peak-Cost}
    \end{subfigure}
    \vspace{-0.1in}
    \caption{\small Comparison of inference performances of PD, RCRL, Surrogate Objective(SO)~\eqref{eq:main_obj} trained on non-perturbed(NP) environment and \ours\ trained in perturbed (P) environment.\ours\ performs best with a peak-cost below threshold 0.1 indicating safe inference and highest returns in a perturbed environment.}
    \label{fig:hc_inference}
\end{figure}

Fig.~\ref{fig:hc_train} and Fig.~\ref{fig:hc_inference} compare the training and inference behavior of \ours, Primal Dual, RCRL~\cite{yu2022reachability} and the non-perturbed Surrogate Objective on the constrained HalfCheetah environment. The results show that \ours\ maintains lower returns while training, and all methods are safe while training. The inference is performed in a perturbed environment with a different perturbation of $\sigma=1.5$ than training ($\sigma=2.0$). \ours\ remains safe with peak-cost constraint under the threshold in perturbed dynamics while still achieving the highest cumulative reward.

These results demonstrate that robustness-aware training improves both performance and safety generalization in MuJoCo continuous-control environments.

\vspace{-0.06in}
\section{Conclusion and Future Work}
\vspace{-0.06in}
In this work, we address the challenges of peak-cost constrained MDPs, demonstrating that they may not exhibit zero duality gap, unlike standard CMDPs. We propose a robust surrogate approach that achieves near-optimal solutions while ensuring constraint satisfaction within an acceptable margin $b$. Through extensive experiments on perturbed CartPole and Mujoco  HalfCheetah environments, we show that \ours\ outperforms existing methods in terms of both constraint satisfaction and reward optimization, even under uncertain dynamics. 

The characterization of the convergence guarantees of the proposed approach has been left for the future. Extending this framework to more complex real-world systems constitutes an important future research direction.

\vspace{-0.06in}
\section*{Acknowledgement}
This work is partially supported by NSF ECCS EPCL Award \#2534262. The work of S M Rajkumar and D Goswami are also  partly supported by NSF-DMS Math-DT under Grant 2529302.

\bibliographystyle{IEEEtran}
\bibliography{reference}
\appendix
\subsection{Proof of Theorem~\ref{thm:non-zero-duality}and Proposition~\ref{prop:equivalence}}\label{sec:proof}
\textit{Formal Statement of Theorem~\ref{thm:non-zero-duality}}:
Consider a discounted MDP with discount factor $\gamma=0.9$, one nonterminal state $s$, and one absorbing terminal state $\bot$. At state $s$, there are two actions:

\begin{itemize}
    \item action $a_0$: reward $r(s,a_0)=4$, cost $c(s,a_0)=0$, and transition
    \[
    P(\bot\mid s,a_0)=1;
    \]
    \item action $a_1$: reward $r(s,a_1)=8$, cost $c(s,a_1)=\tfrac12$, and transition
    \[
    P(s\mid s,a_1)=1.
    \]
\end{itemize}

Let $\pi_p$ denote the stationary randomized policy that chooses $a_0$ with probability $p\in[0,1]$ and $a_1$ with probability $1-p$. Define the hard-max modified constraint value
\begin{align}
V_c^{\pi_p}(s)
=
\sum_{a}\pi_p(a\mid s)\Bigl((1-\gamma)c(s,a) \nonumber\\ + \gamma \max\{c(s,a), \mathbb{E}[V_c^{\pi_p}(s')\mid s,a]\}\Bigr),
\end{align}
with $V_c^{\pi_p}(\bot)=0$.

Consider the constrained optimization problem
\[
P^\star
:=
\min_{p\in[0,1]} J_r(\pi_p)
\qquad\text{s.t.}\qquad
V_c^{\pi_p}(s)\le b,
\]
with threshold $b=0.4$, where $J_r(\pi_p)$ is the standard negative discounted reward value starting from $s$. Here, note that peak cost-constraint across the trajectory is the same as here because of the two states with one state being terminal.

Then the associated Lagrangian dual
\[
D^\star
:=
\sup_{\nu\ge 0}\ \inf_{p\in[0,1]}
\Bigl\{
J_r(\pi_p)-\nu\bigl(V_c^{\pi_p}(s)-b\bigr)
\Bigr\}
\]
satisfies
\[
D^\star<P^\star.
\]
More precisely,
\[
P^\star=-\tfrac{180}{7}\approx -25.7143,
\qquad
D^\star=-64.8,
\]
so that
\[
D^\star-P^\star=-\tfrac{1368}{35}\approx -39.0857<0.
\]
Hence the peak-cost-constraint problem admits a strict duality gap.

\begin{proof}
Let $J_r(\pi_p)$ denote the discounted negative reward value starting from $s$. By the Bellman equation,
\[
J_r(\pi_p)
=
p\cdot (-4) + (1-p)\bigl((-8)+\gamma J_r(\pi_p)\bigr).
\]
Since $\gamma=0.9$, this becomes
\[
J_r(\pi_p)
=
-4p+(1-p)(-8+0.9J_r(\pi_p)).
\]
Rearranging gives
\[
J_r(\pi_p)\bigl(1-0.9(1-p)\bigr)=-(8-4p),
\]
hence
\[
J_r(\pi_p)=-\tfrac{8-4p}{0.1+0.9p}.
\]

Next, let $V(p):=V_c^{\pi_p}(s)$. By definition of the modified hard-max constraint,
\begin{align}
V(p)
&=
p\Bigl((1-\gamma)\cdot 0+\gamma\max\{0,0\}\Bigr)
+ \notag\\
& \qquad  (1-p)\Bigl((1-\gamma)\tfrac12   +\gamma\max\{\tfrac12,V(p)\}\Bigr).
\end{align}
The first term is zero, so
\[
V(p)
=
(1-p)\Bigl(0.1\cdot\tfrac12+0.9\max\{\tfrac12,V(p)\}\Bigr).
\]

We now solve this fixed-point equation. Suppose first that $V(p)\le \tfrac12$. Then
\[
\max\{\tfrac12,V(p)\}=\tfrac12,
\]
and thus
\[
V(p)
=
(1-p)\Bigl(0.05+0.9\cdot\tfrac12\Bigr)
=
(1-p)\cdot\tfrac12
=
\tfrac{1-p}{2}.
\]
Since $\tfrac{1-p}{2}\le \tfrac12$ for all $p\in[0,1]$, this is self-consistent. Therefore,
\[
V_c^{\pi_p}(s)=\tfrac{1-p}{2},\qquad \forall p\in[0,1].
\]

We now solve the primal problem
\[
\min_{p\in[0,1]} J_r(\pi_p)
\qquad\text{s.t.}\qquad
\tfrac{1-p}{2}\le 0.4.
\]
The constraint is equivalent to
\[
1-p\le 0.8
\iff
p\ge 0.2.
\]
Moreover,
\[
J_r(\pi_p)=-\tfrac{8-4p}{0.1+0.9p},
\]
and differentiating yields
\begin{align}
\tfrac{d}{dp}J_r(\pi_p)
=
\tfrac{4(0.1+0.9p)+(8-4p)(0.9)}{(0.1+0.9p)^2}
\nonumber \\ =
\tfrac{7.6}{(0.1+0.9p)^2}>0.
\end{align}
Hence $J_r(\pi_p)$ is strictly increasing in $p$, so the primal optimum is attained at the smallest feasible $p$, namely $p^\star=0.2$. Therefore,
\[
P^\star
=
J_r(\pi_{0.2})
=
-\tfrac{8-4(0.2)}{0.1+0.9(0.2)}
=
-\tfrac{7.2}{0.28}
=
-\tfrac{180}{7}.
\]

Now consider the Lagrangian
\[
L(p,\nu)
=
J_r(\pi_p)-\nu\bigl(V_c^{\pi_p}(s)-0.4\bigr),
\qquad \nu\ge 0.
\]
Substituting the formulas for $J_r(\pi_p)$ and $V_c^{\pi_p}(s)$ gives
\begin{align}
L(p,\nu)
=
-\tfrac{8-4p}{0.1+0.9p}
-
\nu\left(\tfrac{1-p}{2}-0.4\right)
\nonumber \\ =
-\tfrac{8-4p}{0.1+0.9p}
-\nu(0.1-0.5p).
\end{align}

Fix $\nu\ge 0$. The derivative with respect to $p$ is
\[
\tfrac{\partial}{\partial p}L(p,\nu)
=
\tfrac{7.6}{(0.1+0.9p)^2}
+\tfrac{\nu}{2},
\]
and the second derivative is
\[
\tfrac{\partial^2}{\partial p^2}L(p,\nu)
=
-\tfrac{13.68}{(0.1+0.9p)^3}>0.
\]
Thus for every fixed $\nu$, the function $p\mapsto L(p,\nu)$ is concave on $[0,1]$. Therefore its minimum over $[0,1]$ is attained at one of the endpoints:
\[
g(\nu)
:=
\inf_{p\in[0,1]}L(p,\nu)
=
\min\{L(0,\nu),L(1,\nu)\}.
\]

Evaluating the endpoints,
\[
L(0,\nu)=-\tfrac{8}{0.1}-0.1\nu=-80-0.1\nu,
\]
and
\[
L(1,\nu)=-\tfrac{4}{1}+0.4\nu=-4+0.4\nu.
\]
Hence
\[
g(\nu)=\max\{-80-0.1\nu,\;-4+0.4\nu\}.
\]

The maximum of this minimum of two affine functions occurs when they are equal:
\[
-80-0.1\nu=-4+0.4\nu
\iff
-76=0.5\nu
\iff
\nu^\star=-152.
\]
Substituting back,
\[
D^\star=g(\nu^\star)=-80-0.1(-152)=-64.8.
\]

Finally,
\begin{align*}
D^\star-P^\star
=
-64.8+\tfrac{180}{7}
=
-\tfrac{324}{5}+\tfrac{180}{7}
& =
\tfrac{-2268+900}{35} \\
&=
-\tfrac{1368}{35}
<0.
\end{align*}
Thus the duality gap is strict.
\end{proof}

\subsection{Proof of Proposition~\ref{prop:equivalence}}
\label{sec:proof}

\begin{proof}
\textbf{Part (i): Reward Optimality.}
Since $\pi^*$ is feasible for problem~\eqref{eq:robust_rcrl}, we have:
    $\max_{P \in \mathcal{P},\ s \in \mathcal{S}^{\pi^*}} 
    V_{c,\mathrm{peak}}^{\pi^*,P}(s) \leq b,$, thus,
 the surrogate objective evaluated at $\pi^*$ satisfies:
\begin{align}
   & \max\left\{\frac{\max_{P \in \mathcal{P}} J_r^{\pi^*,P}}{\beta},\ 
    \max_{P \in \mathcal{P},\ s \in \mathcal{S}^{\pi^*}} 
    V_{c,\mathrm{peak}}^{\pi^*,P}(s) - b\right\} \nonumber\\
   & = \frac{\max_{P \in \mathcal{P}} J_r^{\pi^*,P}}{\beta},
\end{align}
since the second term inside the max is non-positive. Since $\hat{\pi}^*$ is the 
minimizer of the surrogate problem~\eqref{eq:main_obj}, we have:
\begin{align}\label{eq:inner_opt}
    & \max\left\{\frac{\max_{P \in \mathcal{P}} J_r^{\hat{\pi}^*,P}}{\beta},\ 
    \max_{P \in \mathcal{P},\ s \in \mathcal{S}^{\hat{\pi}^*}} 
    V_{c,\mathrm{peak}}^{\hat{\pi}^*,P}(s) - b\right\} 
    \nonumber\\& \leq \frac{\max_{P \in \mathcal{P}} J_r^{\pi^*,P}}{\beta}.
\end{align}
Since the left-hand side is greater than or equal to 
$\frac{\max_{P \in \mathcal{P}} J_r^{\hat{\pi}^*,P}}{\beta}$, the result follows.

\textbf{Part (ii): Constraint Violation Bound.}
From (\ref{eq:inner_opt}):
\begin{align}
    & \max\left\{\frac{\max_{P \in \mathcal{P}} J_r^{\hat{\pi}^*,P}}{\beta},\ 
    \max_{P \in \mathcal{P},\ s \in \mathcal{S}^{\hat{\pi}^*}} 
    V_{c,\mathrm{peak}}^{\hat{\pi}^*,P}(s) - b\right\} \nonumber\\
    & \leq \frac{\max_{P \in \mathcal{P}} J_r^{\pi^*,P}}{\beta}.
\end{align}
Thus,
\begin{equation}
    \max_{P \in \mathcal{P},\ s \in \mathcal{S}^{\hat{\pi}^*}} 
    V_{c,\mathrm{peak}}^{\hat{\pi}^*,P}(s) - b 
    \leq \frac{\max_{P \in \mathcal{P}} J_r^{\pi^*,P}}{\beta}.
\end{equation}
Since $-r\in [0,1]$, thus, we have 
$J_r^{\pi^*,P} \leq 1/(1-\gamma)$, we obtain:
\begin{equation}
    \max_{P \in \mathcal{P},\ s \in \mathcal{S}^{\hat{\pi}^*}} 
    V_{c,\mathrm{peak}}^{\hat{\pi}^*,P}(s) - b \leq \frac{1}{\beta (1-\gamma)}.
\end{equation}
Setting $\beta \geq (1-\gamma)^{-1}/\epsilon$ completes the proof of part (ii).
\end{proof}

\subsection{Experiment Details}
For CartPole and HalfCheetah, the architecture consists of three networks: the Actor, the Reward Critic, and the Peak-Cost Critic. Each network is implemented as a feedforward neural network with two hidden layers of 64 neurons each and ReLU activations. The Actor outputs the parameters of a Gaussian action distribution for continuous control. The Reward Critic outputs a scalar estimate of the reward value function, while the Peak-Cost Critic outputs a scalar estimate of the peak-cost value function.

\parab{Hyper-parameters for CartPole}
 All Actor and Critic networks are trained with a learning rate of $3 \times 10^{-4}$. For the experiments, we use  $\beta = 25$, baseline $b=2.0$, and apply a warm start of 300 episodes for \ours.\\
\textbf{Hyper-parameter for HalfCheetah}:
For HalfCheetah, all Actor and Critic networks are trained with a learning rate of $5 \times 10^{-4}$. We use a Lagrangian multiplier $\beta = 30000$, $b=0.1$ and a warm start of 500 episodes.


%

\end{document}